\documentclass[sigconf]{acmart}

\AtBeginDocument{%
  \providecommand\BibTeX{{%
    \normalfont B\kern-0.5em{\scshape i\kern-0.25em b}\kern-0.8em\TeX}}}

\usepackage{mymacros} 
\usepackage[hide]{notes}

\newtheorem{theorem}{Theorem}
\newtheorem{problem}{Problem}
\newtheorem{proposition}{Proposition}

\copyrightyear{2020}
\acmYear{2020}
\setcopyright{acmcopyright}
\acmConference[KDD '20]{Proceedings of the 26th ACM SIGKDD Conference on Knowledge Discovery and Data Mining}{August 23--27, 2020}{Virtual Event, CA, USA}
\acmBooktitle{Proceedings of the 26th ACM SIGKDD Conference on Knowledge Discovery and Data Mining (KDD '20), August 23--27, 2020, Virtual Event, CA, USA}
\acmPrice{15.00}
\acmDOI{10.1145/3394486.3403204}
\acmISBN{978-1-4503-7998-4/20/08}

\begin{document}
\fancyhead{}

\title{Diverse Rule Sets}

\author{Guangyi Zhang}
\email{guaz@kth.se}
\affiliation{%
  \institution{KTH Royal Institute of Technology}
  \city{Stockholm}
  \country{Sweden}
}
\author{Aristides Gionis}
\email{argioni@kth.se}
\affiliation{%
  \institution{KTH Royal Institute of Technology}
  \city{Stockholm}
  \country{Sweden}
}

\renewcommand{\shortauthors}{Zhang and Gionis.}

\begin{abstract}
While machine-learning models are flourishing and transforming many aspects of everyday life, 
the inability of humans to understand complex models poses difficulties for these models 
to be fully trusted and embraced.
Thus, interpretability of models has been recognized as an equally important quality as their predictive power.
In particular, rule-based systems are experiencing a renaissance owing to their intuitive if-then representation.

However, simply being rule-based does not ensure interpretability.
For example, overlapped rules spawn ambiguity and hinder interpretation.
Here we propose a novel approach of inferring diverse rule sets, 
by optimizing small overlap among decision rules 
with a 2-approximation guarantee under the framework of Max-Sum diversification.
We formulate the problem as maximizing a weighted sum of discriminative quality and diversity of a rule set.

In order to overcome an exponential-size search space of association rules, 
we investigate several natural options for a small candidate set of high-quality rules, 
including frequent and accurate rules, and examine their hardness.
Leveraging the special structure in our formulation, 
we then devise an efficient randomized algorithm, which samples rules 
that are highly discriminative and have small overlap. 
The proposed sampling algorithm analytically targets a distribution of rules that is tailored to our objective.

We demonstrate the superior predictive power and interpretability
of our model with a comprehensive empirical study against strong baselines.
\end{abstract}

\ccsdesc[500]{Computing methodologies~Rule learning}
\ccsdesc[500]{Computing methodologies~Supervised learning by classification}

\keywords{rule sets; diversification; rule learning; pattern mining; sampling; classifier}


\maketitle

\section{Introduction}\label{sec:intro}
There is a general consensus in the data-science community that interpretability is vital for data-driven models to be understood, trusted, and used by practitioners.
This is especially true in safety-critical applications, 
such as disease diagnosis and criminal justice systems~\cite{lipton2018mythos}.
Rule-based models have long been considered interpretable, 
because rules offer an intuitive representation of knowledge~\cite{han2011data}.
Rule-based models have been 
used as a popular proxy to decompose and explain other complex models~\cite{gilpin2018explaining}.

Some rule-based models are easier to interpret than others.
For example, rule sets (also known as DNF or AND-of-ORs) 
are generally considered easier to interpret than decision lists, 
due to a flatter representation~\cite{freitas2014comprehensible,lakkaraju2016interpretable}.
Recently, there has been an increasing interest in further enhancing interpretability of rule-set models.
Building on work to minimize model complexity~\cite{malioutov2013exact,su2015interpretable,dash2018boolean,wang2017bayesian},
Lakkaraju et al.\ develop Interpretable Decision Sets (\ids)~\cite{lakkaraju2016interpretable}, 
whose key property is small overlap among rules.
Since overlapping rules create ambiguity and need a conflict resolution strategy, 
it is rational to consider small overlap as a new effective criterion to improve interpretability.

It is worth noting that each path in a decision tree can also be seen as a decision rule.
Although these paths have zero overlap, they are strictly organized in a restricted form of a tree, 
and required to cover the entire dataset, which may not be realistic.
Thus, we mainly focus on direct rule set induction in this paper.

\begin{figure}[t] 
\centering
\captionsetup[subfloat]{farskip=2pt,captionskip=1pt}
\subfloat[]{
    \includegraphics[width =.22\textwidth]{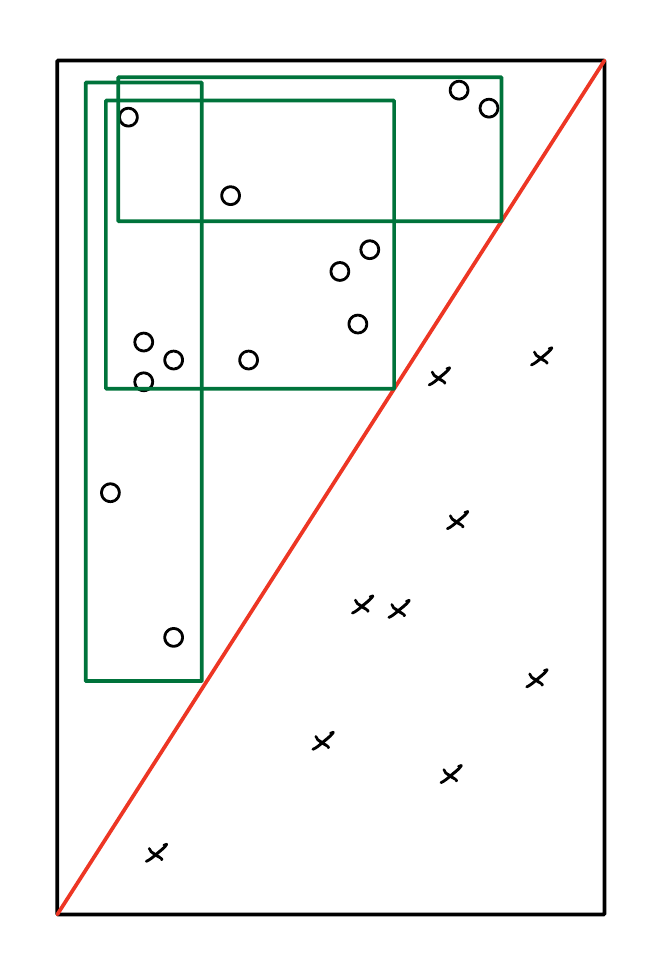}
	\label{fig:retangles:a}}
\subfloat[]{
    \includegraphics[width = .22\textwidth]{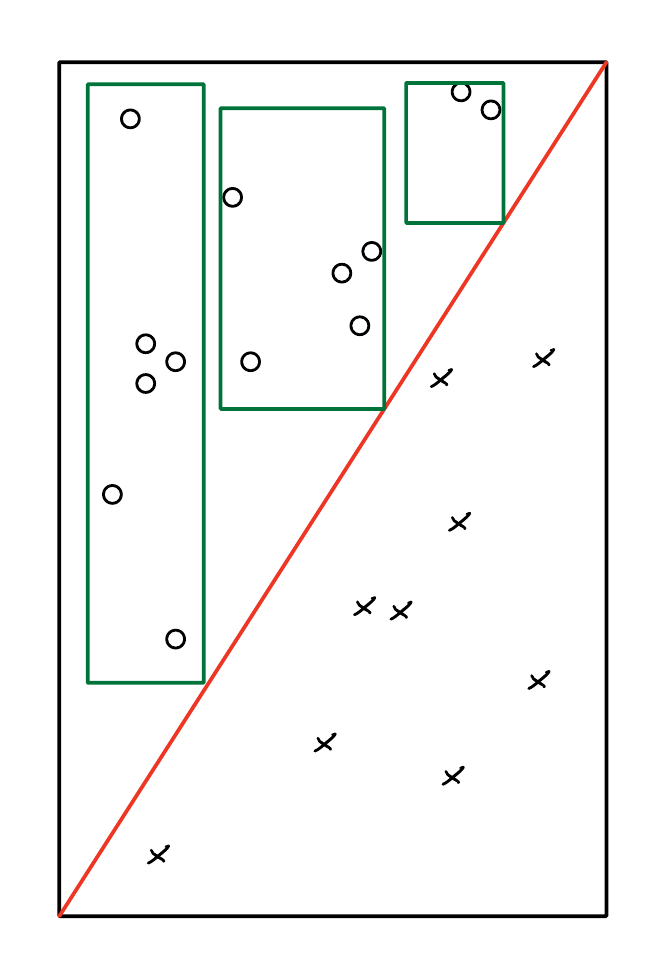}
	\label{fig:retangles:b}}
\caption{A toy example of two different rule decision sets, where a rule is represented by a green rectangle. 
Rule set {\bf (a)} has been produced by a sequential covering algorithm, 
while {\bf (b)} illustrates a diverse rule set.}
\label{fig:retangles}
\Description{A toy example of two different kinds of rule sets.}
\end{figure}

We illustrate the notion of small overlap by a toy example in Figure~\ref{fig:retangles}, 
where rules are represented by green rectangles.
A rule set generated by the popular sequential covering algorithm (\seqcov) \cite{han2011data} 
may consist of a set of largely overlapping rules, 
as the selection of a rule ignores completely the coverage of previous rules.
On the contrary, in an ideal case, an interpretable rule set 
should produce a set of disjoint rules, 
under which, the membership of each instance is unambiguous.
Such behavior is desirable, especially in multi-class classification.

In this paper, we relate small overlap among decision rules to another line of research, diversification~\cite{ravi1994heuristic,gollapudi2009axiomatic,borodin2012max}.
In particular, we concentrate on a type of diversification problems, called \maxsum or remote-clique diversification, where the diversity of a set is defined to be the sum of distances between every pair of points in the set.
\maxsum aims at a selection of points that maximize diversity, under a cardinality constraint.
When the distance function is a metric, there are known 2-approximation algorithms for the \maxsum problem.
Furthermore, an additional monotone non-decreasing submodular function of the selected set, 
referred to as a quality measure 
(also known as a relevance measure in information retrieval area), 
can be incorporated into the framework, with no loss in the approximation ratio.
Thus, diversification can be generalized to maximize a weighted sum of quality and diversity of a fixed-size~set.

Inspired by the diversification framework, 
we adopt a new viewpoint on interpretable rule sets,
where we define the problem as selecting a set of diverse decision rules from a large collection
of candidate rules.
Previous work, such as the \ids method, formulates the problem by using 
a non-normalized non-monotone submodular function, 
which can be approximated deterministically within a factor of 3 
(or within a factor of 2 by using randomization)~\cite{buchbinder2015tight}.
Instead, we propose a novel formulation, 
which can be approximated deterministically within a factor of 2, 
thanks to techniques borrowed from diversification~\cite{borodin2012max}.

Moreover, our approach enjoys several advantages over \ids, 
including superior performance, 
fewer hyper-parameters, 
having linear complexity with respect to the size of the candidate rule set 
(\ids has quadratic complexity), 
and an option to explicitly control the size of selected rules, 
which can be automatically determined, as well. 
To the best of our knowledge, this is the first work to apply a notion of diversification to rule-based~systems.

Our formulation consists of a weighted sum of two components, discriminative quality and diversity.
The objective is to maximize the sum under a cardinality constraint on the rule set.
The proposed quality function encourages a selection of discriminative rules, 
while diversity obviates unnecessary overlap among selected rules.
The selection procedure can be accomplished by a non-oblivious greedy algorithm with a 2-approximation guarantee~\cite{borodin2012max}.

Another contribution we make in this paper, is a sampling algorithm to tackle the long-standing difficulty in rule learning~\cite{furnkranz2012foundations}, or more specifically, in associative classification~\cite{liu1998integrating}, which is the exponential-size search space of candidate rules.
Associative classification classifies an unseen record using a small set of association rules, 
which is usually selected from a larger pre-mined candidate set.
An exemplar of associative classifier is \cba~\cite{liu1998integrating}.

Clearly, the success of associative classification heavily depends on the quality of the candidate set.
The most common approach is to accept a compromise solution of being restricted to frequent rules, which is adopted in \ids.
However, frequent rules are in general not a qualified candidate for an interpretable rule set, as they possess no discriminative power, and capture mostly commonly known, redundant and less interesting patterns.
For example, a frequent rule may consist of a set of uncorrelated frequent conditions (feature-value pairs).
Besides, it has been proved that counting frequent rules is intractable~\cite{gunopulos2003discovering}, let alone enumerating them, especially when conditions (items) are large and dense~\cite{bayardo1999constraint}.
We investigate several natural options for a candidate set, their effectiveness and their hardness, such as accurate rules.

We show that, unfortunately, 
for frequent~\cite{boley2008randomized} and accurate rules, almost-uniform sampling is infeasible.
Despite these negative observations, we harness the special structure of our formulation, 
and propose a novel and efficient sampling algorithm to effectively sample discriminative and less-overlapping rules from an exponential-size search space at each iteration.
The algorithm, inspired by a sampling technique developed by Boley et al.~\cite{boley2011direct}, 
analytically targets a sampling distribution over association rules, 
which is tailored for our optimization objective.

Our contributions are summarized as follows.
\begin{itemize}
	\item To the best of our knowledge, this is the first work to relate diversification to small overlap among decision rules to produce interpretable diverse rule sets.
	\item We propose a novel formulation for a diverse set of rules that are accurate, have small overlap, and can be solved efficiently with a 2-approximation guarantee.
	\item We investigate several options for a candidate rule set in an exponential-size search space, including frequent and accurate rules, and examine their effectiveness and hardness.
	\item We circumvent the difficulty in mining association rules from an exponential-size search space by proposing a novel and efficient randomized algorithm, which samples rules that are discriminative and have small overlap. 
	\item A comprehensive empirical study on various real-life datasets is conducted, and it shows superior predictive performance and excellent interpretability properties of our model against strong baselines.
\end{itemize}

The rest of the paper is organized as follows.
In Section~\ref{sec:related} we discuss the related work.
We delineate our proposal for the diverse rule-set problem in Section~\ref{sec:problem}.
We analyze the hardness of the proposed formulation and of different association-rule mining options in Section~\ref{sec:hardness}, followed by a presentation of our sampling algorithm in Section~\ref{sec:algorithm}.
We evaluate the performance of our algorithm in Section~\ref{sec:experiment}.
Finally, we present our conclusions in Section~\ref{sec:conclusion}.

\section{Related work}\label{sec:related}

\spara{Rule learning}.
Learning theory inspects rule learning from a computational perspective.
\citet{valiant1984theory} introduced PAC learning and asked whether polynomial-size DNF 
can be efficiently PAC-learned in a noise-free setting.
This question remains open, 
and researchers try to attack the problem in restricted forms,
however, these scenarios are less practical for real-world applications with noisy data.

Predominate practical rule-learning paradigms for rule sets 
include sequential covering algorithms~\cite{han2011data}, 
and associative classifiers~\cite{liu1998integrating}.
The former iteratively learns one rule at a time over the uncovered data, 
typically by means of generalization or specialization, 
i.e., adding of removing a condition to the rule body~\cite{furnkranz2012foundations}.
Popular variants include \cn~\cite{clark1989cn2} and RIPPER~\cite{cohen1995fast}.
Associative classifiers use association rules, 
which are usually pre-mined using itemset-mining techniques.
A set of rules is selected from candidate association rules via heuristics~\cite{liu1998integrating} 
or by optimizing an objective~\cite{lakkaraju2016interpretable,angelino2017learning,wang2017bayesian}.
Our method falls into the second paradigm.

For associative classifiers, initial methods select rules from a set of pre-mined model-independent candidates, 
with respect to, for example, frequency, confidence, lift, or other constraints, 
while more recent models embrace an integrated approach~\cite{dash2018boolean,malioutov2013exact}.
Both types of approaches suffer from a computational setback.
Pattern explosion easily renders the former family of methods infeasible, 
while optimization in the latter is inherently hard.
Our method lies at the middle of the two ends of the spectrum, 
where our rule generation is iteratively guided by the model.

\spara{Interpretable rule-based systems}.
Here we only discuss in\-ter\-pretabil\-i\-ty-related properties for rule sets.
Interested readers are referred to the comprehensive survey on interpretability by~\citet{freitas2014comprehensible}.
Most existing models characterize the interpretability of a rule set as having low model complexity.
Early work focuses on the sparsity of rules, 
i.e., a small number of conditions~\cite{malioutov2013exact,su2015interpretable}.
Recent work further considers the size of a rule set, i.e., 
a small number of short rules~\cite{dash2018boolean,wang2017bayesian}.
Lakkaraju et al.\ advocate small overlap among rules as a new criterion 
for interpretability~\cite{lakkaraju2016interpretable}.
Our model emphasizes small overlap, while implicitly taking other criteria into consideration.


\spara{Itemset mining}.
Frequent itemset mining (\fim)~\cite{agrawal1994fast} is one of the most well-known problems in data mining.
Two major drawbacks of frequent itemsets are pattern explosion and lack of interestingness.
A vast amount of literature can be roughly categorized into two groups.
The first group studies efficient data structures and algorithms for \fim 
itself or its condensed representations (bases), 
such as closed itemsets and maximal frequent itemsets.
The second group goes beyond frequency, and puts forth different interestingness measures or constraints 
for an individual pattern~\cite{cheng2007discriminative} or a set of patterns~\cite{knobbe2006pattern}.
A decision rule in our setting can be viewed as a labeled itemset.

\spara{Output-space sampling}.
One approach to tackle prohibitive output space of patterns is via sampling~\cite{chaoji2008origami}.
It is important to distinguish between input-space~\cite{toivonen1996sampling} 
and output-space sampling, where the former performs sampling on database instead of patterns.

\citet{chaoji2008origami} sample maximum frequent subgraphs via a randomization on extension of a path, 
starting from an empty edge set.
Boley et al.\ devise an efficient two-step sampling procedure by first sampling a data record and than sampling a subset of the record, 
so that samples follow a distribution proportional to several interestingness measures, 
such as frequency~\cite{boley2011direct,boley2012linear}.

Some existing works approach output-space sampling using Monte Carlo Markov Chains.
Al Hasan et al. simulate random walks on the frequent pattern partial order, 
and target different stationary distributions via well-designed transition matrices or Metropolis-Hastings algorithms~\cite{al2009musk,al2009output}.
\citet{boley2010formal} construct a sophisticated random work over a concept lattice, 
where each concept corresponds to a closed itemset.
However, none of these random walks guarantee convergence in polynomial time.


\spara{Diversification}.
The maximum dispersion problem was first studied by Ravi et al.~\cite{ravi1994heuristic}.
Gollapudi and Sharma~\cite{gollapudi2009axiomatic} 
turn it into a bi-objective optimization by incorporating a second quality objective.
Borodin et al.~\cite{borodin2012max} extend the formulation to allow a submodular quality function.
Our paper is the first to apply diversification to interpretable rule sets.

\spara{Hardness results in rule mining}.
Compared to the studies for efficient algorithms for association rule mining, 
little attention has been drawn to its computational complexity.
\citet{boros2002complexity} prove that it is \np-hard
to decide whether a given set of maximal frequent itemsets is complete.
\citet{gunopulos2003discovering} prove 
\sharpp-completeness for counting the number of frequent itemsets and 
\np-completeness of mining a \supp-support itemset of given length.
\citet{yang2004complexity} further proves \sharpp-completeness for counting maximal frequent itemsets.
Previous work concentrates on complexity of mining frequent or maximal frequent patterns, 
while our work extend the hardness result to mining accurate rules among labeled itemsets.

Another line of hardness results is directed at inapproximability of sampling and counting of patterns.
The seminal work of \citet{khot2004ruling} proves the inapproximability of Maximum Balanced Biclique (\maxbbip) up to a factor of $\text{size}(x)^\delta$ for instance $x$ and a constant $\delta>0$,
assuming a widely-believed assumption $\np \not\subseteq \cap_{\epsilon>0} \text{BPTIME}(2^{n^\epsilon})$.
Afterwards, \citet{boley2007approximating} introduces a direct reduction to Maximum Frequent Itemset (\maxfs).
Due to the polynomial equivalence between \maxfs and computing its cardinality \maxfc, 
\citet{boley2008randomized} extend the inapproximability to counting frequent itemsets, 
i.e., $|\freqitemset|$.
According to the polynomial equivalence between approximate counting and 
almost-uniform sampling~\cite{jerrum2003counting}, 
almost-uniform sampling of \freqitemset is intractable.
Based on previous work and a reduction from counting $|\freqitemset|$ to counting accurate rules $|\univ_\acc|$, 
we affirm the hardness of almost-uniform sampling of $\univ_\acc$.

\section{Problem formulation}\label{sec:problem}
We describe our method assuming that all features are binary, 
and thus, the input dataset can be seen as a set of labeled transactions.
However, our method is general and can be applied to data with categorical or numerical features.
To represent a dataset as a set of labeled transactions, 
each categorical feature is transformed into one-hot binary features, and 
numerical ones discretized into bins.
More details about binarization of numerical and categorical features 
will be discussed in Section \ref{sec:experiment}.

We are given a set of labeled data records 
$\dat= \{ (\datelem_1,\lbl_1),...,\allowbreak(\datelem_n,\lbl_n) \}$, 
where every data record (or transaction) 
$\datelem\subseteq\feat$ is represented as a vector of binary features, 
$\feat$ is the set of binary features (or items), and 
label $\lbl\in\lbluniv$ is a categorical variable, e.g.,  
$\lbl\in\{0,1\}$ in the binary setting.
The universal set of rules is denoted as $\univ=2^\feat\times\lbluniv$, 
and the set of selected rules as $\sel\subseteq\univ$.
A rule is thus composed of two components, a body and a head, i.e., a set of items and a label, respectively.
A data record $\datelem$ satisfies a rule $\elem$, or equivalently, a rule $\elem$ covers $\datelem$, if $\fbody(\elem) \subseteq \datelem$, also denoted as $\datelem\in\dat(\elem)$.
The rule set $\sel$ can be seen as a multi-class classifier as follows.
\eq{
\sel(\datelem) = 
\begin{cases} 
\flabel(\elem) \eqand \text{if } \text{exists } \elem\in\sel \st \datelem\in\dat(\elem) \eqnl
\text{default label} \eqand \text{otherwise}.
\end{cases}
}
\note[Guangyi]{Just a note here. If a method is a binary classifier, multi-class cannot be simply handled by one-versus-all technique, or otherwise diversity is lost.}
In case that there are several rules in \sel simultaneously covering \datelem, which is precisely the situation our method tries to minimize, we resort to a user-defined conflict resolution strategy, for example, choosing the most accurate rule.
Technically speaking, 
when a class-based ordering is used, \sel is a rule set; 
when a rule-based ordering is used, \sel is a decision list.

We briefly review notions of a set function and a metric here.
A set function $f: 2^U\to\mathbb{R}$ is monotone non-decreasing if $f(B)\ge f(A)$ for all $A\subseteq B\subseteq U$.
Function $f$ is submodular if it satisfies the ``diminishing returns'' property, which means $f(B\cup\{u\})-f(B) \le f(A\cup\{u\})-f(A)$ for $A\subseteq B\subseteq U$ and all $u\in U\setminus B$.
Function $f$ is called supermodular is $-f$ is submodular, and modular if $f$ is both submodular and supermodular.
A metric is a distance function $\dist: U\times U \to \mathbb{R}^+$ such that $\dist(u,u)=0$, $\dist(u,v)=\dist(v,u)$ and $\dist(u,w)\le\dist(u,v)+\dist(v,w)$ for all $u,v,w\in U$.

Before we introduce a formal definition for the problem, 
we will describe the quality and diversity functions for a set of rules.

\spara{Quality.} 
We define the quality of a rule set \sel 
to be the sum of discriminative scores over all rules \elem in $\sel$. 
More formally,
\eq{
\qualityset(\sel) \eqand= \sum_{\elem\in\sel} \quality(\elem), \eqnl
\quality(\elem) \eqand= 
	\sqrt{|\dats{\lbl}(\elem)|} \;
	\mathbbm{1}(\elem) \;
	\kl(P_{\dat(\elem)} || P_{\dat}) 
\label{eq:discrim},
}
\note[Guangyi]{\\
1. An alternative imbalance measure: $\left(\sum_{\pr{\lbl}\in\lbluniv} (|\dats{\pr{\lbl}}(\elem)|/Z(\elem))^2 - \sum_{\pr{\lbl}\in\lbluniv} (|\dats{\pr{\lbl}}|/Z)^2 \right)$, where $Z=|\dat|$, $Z(\elem)=|\dat(\elem)|$.\\
Sum of squares is minimized in a uniform distribution (by Cauchy Schwarz ineq).\\
\begin{align}
\frac 1n = \frac 1n (\sum_i a_i)^2 \le \sum_i a_i^2 \le \sum_i a_i = 1
\end{align}
Max when $a_i=1$ for some $i$, and min when $a_i=a_j$ for all $i,j\in[n]$.\\
2. Sqrt is mainly empirically set. I tried $|\dats{\lbl}(\elem)|$ and it is too big for most datasets.
}
where $\lbl=\flabel(\elem)$ is the class associated with rule \elem, 
$\dats{\lbl}(\elem)$ is the subset of \lbl-labeled data records covered by \elem, 
$\mathbbm{1}(\elem)$ is an indicator function, 
which is equal to 1 if $|\dats{\lbl}(\elem)|/|\dat(\elem)| - |\dats{\lbl}|/|\dat| > 0$ and 0 otherwise, 
$P_{\dat}$ is a distribution of different classes over \dat, 
and
$\kl(\cdot)$ is the KL divergence between two distributions.

In other words, 
the discriminative measure $\quality(\elem)$ 
is positive when the distribution of $\dat(\elem)$ 
deviates from the corresponding distribution on the entire dataset and 
the proportion of data records labeled by $\lbl$ in the cover of \elem is higher than 
the corresponding proportion of data records in the original dataset.

The value of $\quality(\elem)$ is high 
when the distribution $P_{\dat(\elem)}$ (data covered by \elem) is more imbalanced
than the distribution $P_{\dat}$ (the whole dataset), 
which indicates a strong discriminative ability.

Maximizing the KL divergence of an itemset is akin to another problem, 
Exceptional Model Mining (\emm) \cite{leman2008exceptional}, 
which is a generalization of subgroup discovery and 
finds a subgroup description that is substantially different from the complete dataset.
Note that the order of two distributions in KL divergence is important, because KL divergence is not symmetric.
With the chosen order, we encourage rules that take care of a rare class.
Furthermore, there exists a well-known trade-off between accuracy and coverage size in rule learning.
In order to encourage rules of large coverage, 
we multiply the KL divergence with the square root of $|\dats{\lbl}(\elem)|$.
A multiplication of discrimination and coverage can also be found in other popular rule learning methods, 
such as RIPPER~\cite{cohen1995fast}.
It is easy to see that  
$\qualityset(\sel)$ is a non-negative modular set function, and thus, 
monotone non-decreasing submodular.

\spara{Diversity.} 
Diversity is defined as a sum of pairwise distances of elements in the set.
The distance between two rules is defined to be the \emph{Jaccard distance} of their coverage over all data records. 
Here we overload the notation of $\dist(\cdot)$ with the distance function between any two rules, 
$\dist(\elem_i,\elem_j)$.
Note that the Jaccard distance is a metric, and 
maximizing the Jaccard distance captures well the notion of small overlap.
Unpleasant small redundant rules may also cause a high diversity value, but other components in our algorithm collaborate to prevent selecting such rules.
\eq{
\dist(\sel) \eqand= \sum_{\elem_i,\elem_j\in\sel} \dist(\elem_i,\elem_j),\eqnl
\dist(\elem_i,\elem_j) \eqand= 1 - \frac{|\dat(\elem_i) \cap \dat(\elem_j)|}{|\dat(\elem_i) \cup \dat(\elem_j)|}.
}
\note[Guangyi]{Caveat: Jaccord distance is likely to be small if one rule is very big, i.e., it enlarges the denominator.
E.g., $\dist(\sel)$ is small if a rule set has a big rule and many small rules that overlap with the big one.}

\spara{Problem definition.} 
We are now ready to formally define the problem of finding 
a diverse rule of sets,
which is the focus of this paper.

\begin{problem}[Diverse rule set (\dds)]
\label{problem:dds}
Given a set of labeled data records 
$\dat= \{ (\datelem_1,\lbl_1),...,(\datelem_n,\lbl_n)\}$, 
a budget \nsel on the number of rules, 
and a non-negative number \lambdaq, 
we want to find a set of rules \sel
that maximizes the \maxsum diversification function 
\eq{
\distq(\sel) = \qualityset(\sel) + \lambdaq\dist(\sel), \label{eq:obj}
}
over the space of rules \univ and under a cardinality constraint $|\sel|\le\nsel$.
Here, $\lambdaq$ is a user-defined hyper-parameter to control the trade-off between 
the quality function $\qualityset(\cdot)$ and the diversity function $\dist(\cdot)$.
\end{problem}

\spara{Discussion.}
In general, objective functions in the form of Equation~(\ref{eq:obj}) can be approximated within a factor 2 via a greedy algorithm \cite{borodin2012max}, 
provided the $\dist(\cdot)$ is a metric and $\univ$ can be enumerated in polynomial time.
We will address the problem of \univ later.

Maximizing Equation (\ref{eq:obj}) 
will give a desirable rule set in terms of our goal, 
i.e., accuracy and small overlap.
Though $\dist(\cdot)$ may potentially favor small-coverage rules 
since they are less likely to overlap, $\qualityset(\cdot)$ 
counteracts this effect by encouraging large-coverage rules.
Moreover, a rule sampling algorithm described in Section \ref{sec:algorithm} is designed to sample large-coverage candidate rules only among uncovered records, which simulates a similar rule discovery process as \seqcov.
Hence these two measures and the sampling algorithm collaborate to produce a rule set that is both accurate and diverse.

Our model is an inherent multi-class classifier.
For the sake of simplicity, we focus on binary classification in the rest of the paper, and discuss extensions to a multi-class case when needed.

We conclude this section with a clarification on decision rules in theory and in practice.
It is important to discern the discrete nature of a decision rule in our situation.
We exemplify decision rules in an itemset lattice in Figure \ref{fig:lattice}.
A decision rule whose body is an itemset covers all of its supersets on the lattice.
For example, the rule represented by the green oval has a body itemset ``AB'', and it covers both itemsets ``AB'' and ``ABC''.
Thus, in theory, every rule overlaps at the complete itemset (e.g., ``ABC'' in the figure).
Hence, an ideally ``disjoint'' rule set does not exist.
However, in a common case, multiple items are generated by binarization of one feature, so such a complete itemset never appears.

\begin{figure}[t] 
\centering
\includegraphics[width=.25\textwidth]{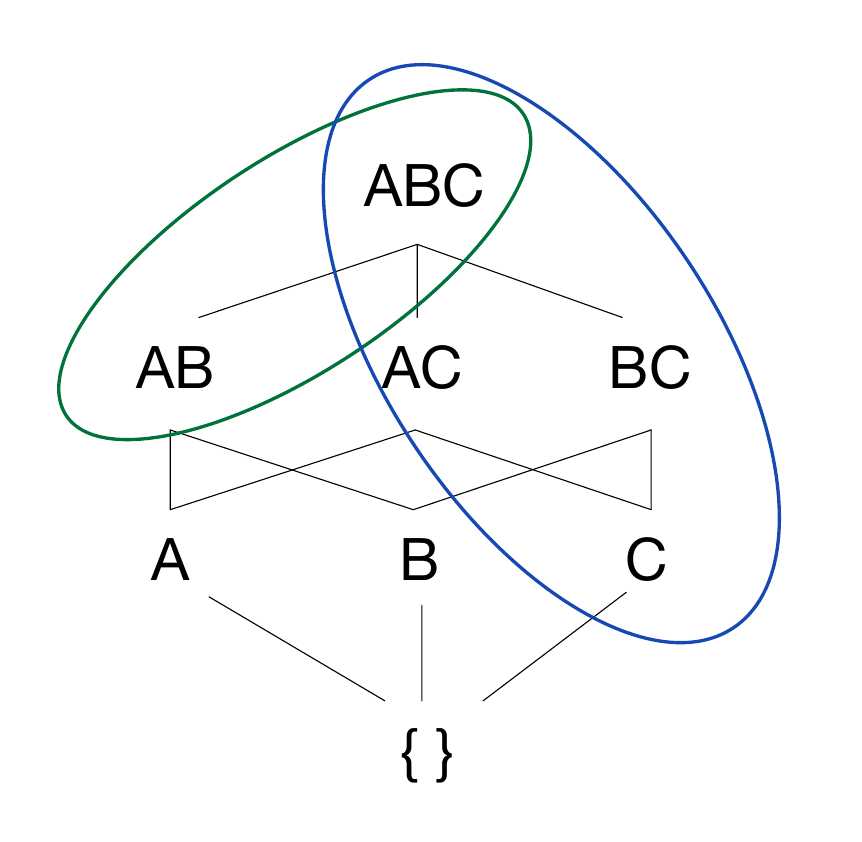} 
\caption{A perspective from a itemset lattice on decision rules. Each rule is represented by an oval.} \label{fig:lattice}
\Description{A perspective from a itemset lattice on decision rules.}
\end{figure}

\section{Hardness}\label{sec:hardness}
The hardness of the \dds problem originates from two aspects, 
($i$) optimizing the objective function in Equation (\ref{eq:obj}), and 
($ii$) mining the optimal rule with respect to the discriminative measure in Equation (\ref{eq:discrim}).
We justify our modeling decisions and along the way we introduce these hardness results.

First, optimizing exactly the objective function in Equation (\ref{eq:obj}) should not be expected, 
and a 2-approximation guarantee is proven to be tight under a standard complexity conjecture.
\begin{proposition}[\cite{borodin2012max}]
Optimizing a general diversification function in the form of Equation (\ref{eq:obj}) is \np-hard.
2-approximation is tight assuming that the planted clique conjecture is true.
\end{proposition}
Greedy and local-search algorithms are both identified to be capable of giving a 2-approximation guarantee.
However, the discriminative measure used in the first term of Equation (\ref{eq:obj}) poses a new difficulty to the approximation algorithms.
As we mentioned in Section \ref{sec:problem}, maximizing the KL divergence in the discriminative measure \quality is a special case of the Exceptional Model Mining (\emm) problem \cite{leman2008exceptional}, 
for which, to the best of our knowledge, there are only heuristic solutions so far.
Obviously, the difficulty lies in the exponential size of the search space \univ .
The approximation algorithms are thus unable to finish in polynomial time as they need to 
perform exhaustive search to find the most discriminative rule.
Therefore, 
to address this challenge, 
we replace the vast search space \univ with a smaller candidate set \cndd.

Most existing methods circumvent the intractable \univ by limiting \cndd to only frequent rules, whose coverage ratio in the dataset is above some threshold.
This strategy has two drawbacks.
First, it is already known that the output space of frequent rules can be exponentially large, i.e., the notorious pattern explosion phenomenon \cite{gunopulos2003discovering}.
In fact, it is \sharpp-hard to count the number of frequent itemsets.
This is true even for a condensed representation of frequent rules, such as maximal frequent itemset \cite{yang2004complexity}.
Second, frequent rules alone are unlikely to provide useful candidates for interpretability and classification purposes.
They usually capture the most commonly-known patterns and are not endowed with discriminative power.

Another natural idea is to permit only accurate rules.
We now show that this is also hard in general.
\begin{proposition}\label{prop:counting-acc}
The problem of counting the number of \acc-accuracy rules for a given $\acc\in[0,1]$, is \sharpp-hard.
\end{proposition}
\note[Guangyi]{We could not argue this statement for any $\acc\in[0,1]$. When $\acc=0$ the problem is trivial.}
\begin{proof}
We show that counting the number of \acc-accuracy rules is at least as hard as counting the number of \supp-support itemsets, for a given $\supp\in\mathbb{N}$.
The proof immediately follows the fact that counting the number of \supp-support itemsets is \sharpp-hard~\citep{gunopulos2003discovering}.

Given a set of $n$ transactions, we create a labeled dataset. 
We set all transactions to be positive data records and include one extra negative data record, i.e., $\dat= \{ (\datelem_1,1),...,(\datelem_n,1),(\datelem_{n+1},0) \}$.
In particular, the negative data record $\datelem_{n+1}$ includes all items.
If we were able to count the number of positive \acc-accuracy rules, by setting $\acc=\frac{\supp}{\supp+1}$, we would be able to count the number of \supp-support itemsets.
\end{proof}

One may suggest an alternative definition for \cndd, such as top-$k$ accurate rules.
Unfortunately, this is still very hard, and does not bear a resemblance to finding a top frequent rule, i.e., the empty itemset, because accuracy does not enjoy a handy monotone property as frequency does \cite{agrawal1994fast}.
The following result affirms the difficulty in finding accurate rules.
It is computationally hard to find the most accurate rule with a given rule size.
\note[Guangyi]{Though it seems to not exclude the possibility of finding the most accurate rule regardless of the rule size.}
\note[Guangyi]{An accurate itemset may not have any accurate itemset child. E.g., Adding any item into an accurate itemset may lead to empty coverage.}
\begin{proposition}
The problem of deciding if there exists a \acc-accuracy rule with at least $t$ items is \np-complete.
\end{proposition}
\begin{proof}
It is easy to see that this problem is in \np.
To show the \np-completeness, we prove via a polynomial time reduction from the Balanced Bipartite Clique problem (\bbc), which has been proved to be \np-complete \citep{garey2002computers}.
Given a bipartite graph $G=(V_1,V_2,E)$, a balanced clique of size $k$ is a complete bipartite graph  $\pr{G}=(\pr{V_1},\pr{V_2},\pr{E})$ with $|\pr{V_1}|=|\pr{V_2}|=k$ and $|\pr{E}|=k^2$.
The \bbc problem is to check the existence of a balanced clique of a given size $k$ in a given bipartite graph~$G$.

Given a bipartite graph $G=(V_1,V_2,E)$ and an integer $k$, 
we create a set of $|V_2|$ items, and a dataset consisting $|V_1|$ positive data records and one extra negative data record.
In particular, a positive data record contains only items that are connected to it, and the negative data record contains all items.
If we were able to find a \acc-accuracy rule with at least $t$ items, by setting $\acc=\frac{k}{k+1}$ and $t=k$, we would be able to find a balanced clique of size $k$.
This completes the reduction.
\end{proof}

We present the last hardness result that rules out the possibility of almost-uniform sampling of accurate rules.
\begin{proposition}
Approximate counting of \acc-accurate rules, i.e., $|\univ_\acc|$, and almost-uniform  sampling of $\univ_\acc$ are both hard, assuming $\np \not\subseteq \cap_{\epsilon>0} \text{BPTIME}(2^{n^\epsilon})$.
\end{proposition}
\begin{proof}
The seminal work of \citet{khot2004ruling} proves the inapproximability of Maximum Balanced Biclique (\maxbbip) up to a factor of $\text{size}(x)^\delta$ for instance $x$ and some constant $\delta>0$, assuming a widely believed assumption $\np \not\subseteq \cap_{\epsilon>0} BPTIME(2^{n^\epsilon})$.
\citet{boley2008randomized} extend the inapproximability to counting frequent itemsets, i.e., $|\freqitemset|$.
According to the polynomial equivalence between approximate counting and almost-uniform  sampling \cite{jerrum2003counting}, almost-uniform  sampling of \freqitemset is intractable.
The hardness of approximate counting of accurate rules, i.e., $|\univ_\acc|$, and almost-uniform  sampling of $\univ_\acc$ follows the reduction from counting \freqitemset to counting $\univ_\acc$, 
as shown in Proposition \ref{prop:counting-acc}.
\end{proof}

\section{Algorithm}\label{sec:algorithm}
The major difficulty of approximating the objective lies in the exponential size of the search space \univ of rules.
A natural idea is to focus on a small candidate subset \cndd of \univ, so that our approximation algorithms can realize an approximation guarantee on a problem instance with an input set \cndd.
Clearly, the quality of the final rule set \sel depends on the quality of \cndd.
We have discussed several popular options for \cndd and their hardness in Section~\ref{sec:hardness}.
Unfortunately, neither frequent rules, nor accurate rules, nor any other candidates that rely on enumerating them are infeasible.

When enumeration is intractable, the simplest scheme is a premature abortion after gathering some desired number of candidates.
This approach is unable to assure representativeness of the chosen candidates, and inevitably introduces a bias that is closely related to the order of searching, such as DFS or BFS in the search of frequent itemsets.
Sampling offers an excellent solution to this dilemma, as randomization secures representativeness in a statistical sense.
However, as we show in Section~\ref{sec:hardness}, almost-uniform sampling of frequent or accurate rules is not possible.
Fortunately, these results do not eliminate the possibility of sampling rules 
from a distribution proportional to some interestingness measure of each rule, 
such as frequency \cite{boley2011direct}.

We devise a sampling distribution that is tailored to our objective, inspired by the mechanism of the greedy approximation algorithm.
In a greedy algorithm, for the first selected rule, the algorithm seeks a rule $\elem$ with the highest $\quality(\elem)$.
In later iterations, the algorithm tries to find a rule \elem, 
which not only enjoys a high $\quality(\elem)$ but it also causes less overlap with previous rules.
This amounts to finding a rule that has a high value for the following measure:
\eq{
\qualitys(\elem) \eqand= |\datplus_i(\elem)| 
\left(|\datminus_i|-|\datminus_i(\elem)|\right) 
\left(|\datun_i|-|\datun_i(\elem)|\right), \label{eq:obj-sample}
}
where $\flabel(\elem)$ is assumed to be positive, the index $i$ stands for the $i$-th iteration, 
$\datplus_i$ and $\datminus_i$ for uncovered positive and negative data records, 
respectively, and $\datun_i$ for covered data records by previously-selected rules.
The measure $\qualitys$ takes larger values for discriminative rules that have large coverage in 
$\datplus_i$ and small coverage in others.
The measure \qualitys thus shares similar properties with the discriminative measure in 
Equation (\ref{eq:discrim}), as they encourage both discriminative and large-coverage rules.

One key difference is that the measure \qualitys partitions the whole dataset into three parts instead of two, 
treating covered data records as a second undesirable ``negative'' dataset that a sampled rule tries to avoid covering.
Therefore, this sampling measure also takes diversity into account, 
and fits seamlessly into the diversification framework.
In the multi-class case, we can sample rules multiple times, 
each time with a different label being the positive label and all others the negative.
\note[Guangyi]{$\datplus_i,\datminus_i,\datun_i$ need to be disjoint.}

A remarkable sampling technique developed by \citet{boley2011direct} inspires us to sample a rule from a distribution proportional to Equation (\ref{eq:obj-sample}) very efficiently.
The main idea of the sampling technique is a two-step sampling procedure, where we first sample one or more data records, and then sample a rule from the power set of a combination of the sampled records.
In the work of \citet{boley2011direct}, they utilize this technique to sample frequent or discriminative rules, while we generalize it beyond two classes of data records in a meaningful application.
We present the proposed sampling algorithm in Algorithm \ref{alg:sampling}.
\begin{theorem}
Algorithm \ref{alg:sampling} samples rules from a distribution proportional to the measure 
in Equation (\ref{eq:obj-sample}).
\end{theorem}
\begin{proof}
Let $\mathcal{L}=\datplus\times\datminus\times\datun$.
For an arbitrary positive rule \elem, we specify a set of potential triples, from which \elem can be sampled.
\eq{
\mathcal{L}(\elem) = \{ \datelemplus,\datelemminus,\datelemun \in \mathcal{L} \mid \elem\subseteq\datelemplus,\elem\not\subseteq\datelemminus,\elem\not\subseteq\datelemun\}.
}
By focusing only on subsets of \datelemplus and in the meantime excluding subsets of $\datelemminus,\datelemun$, we ensure that \elem has a chance of being sampled only in triples in $\mathcal{L}(\elem)$.
The cardinality of $\mathcal{L}(\elem)$ coincides with its measure in Equation (\ref{eq:obj-sample}).
\eq{
|\mathcal{L}(\elem)| = 
|\datplus(\elem)|
\left(|\datminus_i|-|\datminus_i(\elem)|\right)
\left(|\datun_i|-|\datun_i(\elem)|\right) = 
\qualitys(\elem).
}
Therefore, we only need to sample a triple $(\datelemplus,\datelemminus,\datelemun)$ from a record distribution with a probability proportional to 
$2^{|\datelemplus \setminus (\datelemplus\setminus\datelemminus\setminus\datelemun)|}$
$(2^{|\datelemplus\setminus\datelemminus\setminus\datelemun|}-1)$, 
and then uniformly sample a concatenation of two random elements from $\powerset(\datelemplus\setminus\datelemminus\setminus\datelemun) \setminus \emptyset$ and $\powerset(\datelemplus \setminus (\datelemplus\setminus\datelemminus\setminus\datelemun))$, respectively, to ensure that every valid \elem in this triple is exposed equally within this triple and across different triples.
Let \elemr denote a random variable of a rule sampled from the sampling distribution in Algorithm \ref{alg:sampling}.
\begin{align*}
& \Pr(\elemr = \elem) \\ 
& \propto \sum_{(\datelemplus,\datelemminus,\datelemun)\in\mathcal{L}(\elem)} \frac{(2^{|\datelemplus\setminus\datelemminus\setminus\datelemun|}-1) 2^{|\datelemplus \setminus (\datelemplus\setminus\datelemminus\setminus\datelemun)|}}{|\powerset(\datelemplus\setminus\datelemminus\setminus\datelemun) \setminus \emptyset||\powerset(\datelemplus \setminus (\datelemplus\setminus\datelemminus\setminus\datelemun))|} \\
& = |\mathcal{L}(\elem)|.
\end{align*}
\end{proof}

\begin{algorithm}
\raggedright
\small
\caption{Sampling candidate rules}\label{alg:sampling}
\hspace*{\algorithmicindent} \textbf{Input:}
Uncovered positive dataset \datplus,
uncovered nagative dataset~\datminus,
covered dataset \datun, and
number of samples \nsamp.
\\
\hspace*{\algorithmicindent} \textbf{Output:} 
A set of sampled rules $\samp=\{\elem_1,\ldots,\sel_{\nsamp}\}$.
\begin{algorithmic}[1]
\Function{SampleCandidateRules}{$\datplus,\datminus,\datun,\nsamp$}
	\State $s\gets 0$
	\For {$(\datelemplus,\datelemminus,\datelemun)$ in $\datplus\times\datminus\times\datun$} \label{step:product}
	\State $\text{idx} \gets (2^{|\datelemplus\setminus\datelemminus\setminus\datelemun|}-1) * 2^{|\datelemplus \setminus (\datelemplus\setminus\datelemminus\setminus\datelemun)|}$
	\If{$\text{idx}=0$}
		\State Continue
	\EndIf
	\State Associate tuple $(\datelemplus,\datelemminus,\datelemun)$ with tuple $(s, s + \text{idx})$
	\State $s \gets s + \text{idx}$
	\EndFor
	
	\For {$j=1,\ldots,\nsamp$} 
	\State Generate a random integer $r$ from $[0,s)$ u.a.r.
	\State Find the tuple $(\datelemplus,\datelemminus,\datelemun)$ associated with the tuple $(a,b)$ such that $r\in [a,b)$
	\State Sample $\elem_j$ from $(\powerset(\datelemplus\setminus\datelemminus\setminus\datelemun) \setminus \emptyset) \times \powerset(\datelemplus \setminus (\datelemplus\setminus\datelemminus\setminus\datelemun))$ u.a.r.
	\EndFor
	
	\State {\bf Return} $\samp=\{\elem_1,\ldots,\sel_{\nsamp}\}$
\EndFunction
\end{algorithmic}
\end{algorithm}

\smallskip
With all components ready, we now present the full learning algorithm in Algorithm \ref{alg:decisionset}.
Running the greedy algorithm over a different candidate set in each iteration does not guarantee an approximation ratio.
Thus, we run the algorithm a second time over the full candidate set $\cndd_a$.
The set $\cndd_a$ does not pose a problem because small redundant rules are unlikely to be sampled in the first place.
We compare both rule sets from the first and second runs in experiments.
For the ease of exposition, we only sample positive rules in the algorithm, and it is natural to accept the label for the majority class in the dataset as the default label.
In practice, it is also recommended to sample rules for each class, and take the under-represented class as the default class.
\begin{theorem}
Algorithm \ref{alg:decisionset} is a 2-approximation algorithm with respect to the objective function in 
Equation (\ref{eq:obj}) over the search space of all sampled rules.
\end{theorem}
\begin{proof}
The approximation guarantee of a non-oblivious greedy algorithm is proved in the work of \citet{borodin2012max}.
A difference in our case is that a new candidate set \cndd is sampled in each iteration.
In order to maintain an approximation guarantee, we have to run the greedy algorithm a second time in case there exists a better selection \sel over the full candidate set $\cndd_a$.
Thus the better solution between the first run and the second run maintains a 2-approximation guarantee of the algorithm.
\end{proof}

The running time of the sampling algorithm and the full algorithm is 
$\bigO(|\dat|^3 \nfeat + \nsamp(\nfeat + \log |\dat|))$ and
$\bigO(\nsel(|\dat|^3 \nfeat + \nsamp(\nfeat + \log |\dat|)) + \nsel^2\nsamp|\dat| + \nsel^2\nsamp (\nsel|\dat|))$, respectively, where \nsamp is the number of sampled rules in each iteration, and $\nsel=|\sel|$.
A major downside is the cubic running time to the size of the dataset, 
which originates from a preprocessing step (step \ref{step:product}) 
that enumerates the Cartesian product of three partitions of the dataset.

An MCMC technique can be utilized to reduce the cubic complexity to a linear one \cite{boley2012linear}.
For simplicity, the sampling algorithm can be run upon a small sample dataset instead of the full dataset.
Note that a sampled dataset change the absolute magnitude of the sampling measure \qualitys, but the relative ratio remains the same.
For example, feeding a sample $\pr{{\datminus}}$ from \datminus to Algorithm \ref{alg:sampling} samples with respect to the following measure:
\eq{
\pr{\qualitys}(\elem) \eqand= |
\datplus(\elem)|
\left(|\pr{{\datminus}}|-|\pr{{\datminus}}(\elem)|\right),
}
where $\expt[|\pr{{\datminus}}|-|\pr{{\datminus}}(\elem)|] \propto \ndatminus-|\datminus(\elem)|$.
We can further reduce the cost by resorting to an alternative sampling measure,
\eq{
\qualityss(\elem) = |\datplus_i(\elem)| 
\left(|\datminus_i|-|\datminus_i(\elem)| + |\datun_i|-|\datun_i(\elem)|\right) \label{eq:obj-sample2},
}
which shares similar characteristics with the measure \qualitys, 
and can be sampled in a similar way, while requires only quadratic complexity.
In this paper we adopt the measure in Equation (\ref{eq:obj-sample2}) for our experiments.
With these efficient extensions, assuming the sampled dataset is at most of size $c$, 
the complexity reduces to $\bigO(c^3 \nfeat + \nsamp(\nfeat + \log c))$ and
$\bigO(\nsel(c^3 \nfeat + \nsamp(\nfeat + \log c)) + \nsel^2\nsamp|\dat| + \nsel^2\nsamp (\nsel|\dat|))$.

\begin{algorithm}
\raggedright
\small
\caption{Diverse rule set (\dds)}\label{alg:decisionset}
\hspace*{\algorithmicindent} \textbf{Input:}
Positive dataset \datplus,
nagative dataset \datminus,
\\
\hspace*{\algorithmicindent} \textbf{Output:} 
A set of decision rules $\sel=\{\elem_1,\ldots,\sel_{\nsel}\}$.
\begin{algorithmic}[1]
\State $\sel\gets\{\}, \cndd_a\gets\{\}$
\Repeat
	\State $\datun = \datplus(\sel) \cup \datminus(\sel)$
	\State $\cndd = \Call{SampleCandidateRules}{\datplus\setminus\datun,\datminus\setminus\datun,\datun,\nsamp}$
	\State Select a rule $\elem^* = \arg\max_{\elem\in\cndd} \frac 12 \qualityset(\sel\cup\{\elem\}) + \lambdaq \dist(\sel\cup\{\elem\})$
	\State $c \gets (\datplus(\sel\cup\{\elem^*\}) - \datplus(\sel)) / \datplus$
	\State $\sel\gets\sel\cup\{\elem^*\}$
	\State $\cndd_a\gets\cndd_a\cup\cndd$
\Until{$c\ge\epsilon$}
\State $\nsel = |\sel|, \pr{\sel}\gets\{\}$
\For{$j=1\ldots \nsel$}
	\State Select a rule $\elem^* = \arg\max_{\elem\in\cndd_a} \frac 12 \qualityset(\pr{\sel}\cup\{\elem\}) + \lambdaq \dist(\pr{\sel}\cup\{\elem\})$
	\State $\pr{\sel}\gets\pr{\sel}\cup\{\elem^*\}$
\EndFor
\If{ $\distq(\sel) < \distq(\pr{\sel})$}
	\State {\bf Return} $\pr{\sel}$
\EndIf

\State {\bf Return} $\sel$
\end{algorithmic}
\end{algorithm}

\section{Experiments}\label{sec:experiment}
We evaluate the performance of our algorithm from three different perspectives ---
sensitivity to hyper-parameters, predictive power, and interpretability.
Hyper-parameters in our algorithm can be determined in a straightforward way without additional tuning, as we describe in Appendix \ref{ap:param}.
The implementation of our model and all baselines is available online.%
\footnote{\href{https://github.com/Guangyi-Zhang/diverse-rule-set}{https://github.com/Guangyi-Zhang/diverse-rule-set}}

We compare our model \dds with three baselines --- \cba \cite{liu1998integrating}, \cn \cite{clark1989cn2}, and \ids \cite{lakkaraju2016interpretable}.
Models \cba and \ids use frequent rules as candidates.
\cn is learned in a sequential covering approach.
In particular, \ids optimizes small overlap among rules.
Therefore, the baselines serve as excellent opponents to examine two important aspects of our model: predictive power and diversity.
The metrics include balanced accuracy (bacc), ROC AUC (auc), average diversity (div), the number of distinct data records that are covered by more than one rule (overlap).
All metrics are averaged over multiple runs.

Various real-life datasets are tested against the models, which are listed in Table \ref{tab:dataset}.
In the data preprocessing phase, each categorical feature is transformed into one-hot binary features, and numerical ones are each uniformly cut into five bins.
We treat missing value as a separate feature value.
\begin{table}
  \caption{Datasets characteristics: $\text{n}=|\dat|$; $\text{ncls}=|\lbluniv|$;
  $\text{imbalance}=\max_{\lbl\in\lbluniv} |\dats{\lbl}| / \min_{\lbl\in\lbluniv} |\dats{\lbl}|$;
  and \nfeat is the size of binary features.}
  \label{tab:dataset}
\begin{tabular}{lrrrr}
\toprule
{} &        n & ncls	&    imbalance &  \nfeat \\
dataset    &          &        &        \\
\midrule
iris       &    150 &3	&   1.00 &   20 \\
contracept &   1473 &3	&   1.89 &   41 \\
cardio     &   2126 &3	&   9.40 &  126 \\
anuran     &   7195 &4	&  65.00 &  132 \\
avila      &  16348 &4	&   5.15 &   60 \\
\bottomrule
\end{tabular}
\end{table}

\subsection{Sensitivity to \lambdaq}
In this subsection, we investigate the effect of the key hyper-pa\-ram\-eter, the trade-off term \lambdaq between discriminative quality and diversity.
We fix other parameters, and examine the performance of our model under three different magnitudes for \lambdaq, directly calculated from the data (see Appendix \ref{ap:param}).
As shown in the Table \ref{tab:lambda}, best ROC AUC values of the testing datasets concentrate at larger values of \lambdaq, and the largest diversity always lies in the \lambstrict mode.
This confirms the fact that a moderate amount of diversity in a rule set improves their overall performance.

The two objectives, discriminative quality and diversity, are mutually beneficial to each other.
Their relationship is akin to that between relevance and non-redundancy in the information retrieval area.
Multiple similar discriminative rules create redundancy and do not enhance generalization of the decision-making system.
On the other hand, a higher diversity is at the expense of higher model complexity, i.e., a slightly larger number of rules and conditions per rule.
We will fix the mode to be \lambstrict for the rest of our experiments.

\begin{table}
  \caption{Sensitivity to \lambdaq.}
  \label{tab:lambda}
\begin{tabular}{llrrrrr}
\toprule
dataset & \lambdaq & $n_{\text{rules}}$ &  $n_{\text{conds}}$ &   bacc &    auc &   div \\
\midrule
iris	&0    & \bf{8.00} & \bf{1.62} &      0.87 &      0.90 &      0.66 \\
iris	&mean &      9.00 &      2.00 &      0.87 &      0.90 &      0.98 \\
iris	&max  &      8.67 &      2.02 & \bf{0.89} & \bf{0.92} & \bf{1.00} \\\midrule
contracept	&0    & \bf{3.33} & \bf{3.64} & \bf{0.36} & \bf{0.52} &      0.90 \\
contracept	&mean &      6.33 &      4.51 &      0.35 &      0.51 &      0.98 \\
contracept	&max  &      3.67 &      4.02 &      0.35 &      0.51 & \bf{0.99} \\\midrule
cardio	&0    & \bf{21.67} & \bf{7.85} &      0.50 &      0.66 &      0.96 \\
cardio	&mean &      45.67 &      8.67 & \bf{0.71} & \bf{0.78} & \bf{1.00} \\
cardio	&max  &      53.00 &      8.99 &      0.64 &      0.74 & \bf{1.00} \\\midrule
anuran	&0    & \bf{31.67} & \bf{8.52} &      0.50 &      0.70 &      0.87 \\
anuran	&mean &      41.33 &      8.87 & \bf{0.81} & \bf{0.88} & \bf{1.00} \\
anuran	&max  &      40.00 &      8.99 &      0.80 &      0.87 & \bf{1.00} \\\midrule
avila	&0    &      4.00 & \bf{4.72} & \bf{0.26} & \bf{0.51} &      0.59 \\
avila	&mean & \bf{3.67} &      4.98 & \bf{0.26} & \bf{0.51} & \bf{1.00} \\
avila	&max  & \bf{3.67} &      5.32 & \bf{0.26} & \bf{0.51} & \bf{1.00} \\
\bottomrule\end{tabular}
\end{table}

\subsection{Predictive power}
We discuss configurations for running predictive models in Appendix \ref{ap:param}.
One great advantage of our model \dds is that no hyper-parameter tuning is required.
Although it indeed has several hyper-parameters, all of them can be determined automatically.
The complete result is shown in Table \ref{tab:predictive}.
Overall, in terms of ROC AUC values, our model exceeds all baselines by a large margin in many datasets.
In some difficult datasets, our model stops early with only a few essential rules chosen, as the marginal gain of additional rules is too tiny.
We support this claim by conducting another experiment, the results of which are shown in Figure \ref{fig:k}, 
where the maximum number of selected rules is limited to different~values.
A general pattern is that our model selects the most influential and large-coverage rules at the beginning, which justifies the practice of manually setting a maximum number of rules if needed.
It is reasonable to conclude that our model fares well in a predictive task.

In terms of diversity, our model outperforms all baselines in all datasets.
Metric div may be biased and fails to reflect genuine diversity, especially when there are many small-coverage rules like in \cn.
Therefore, the metric overlap serves as a more qualified metric for diversity, and our model achieves significant lower overlap.
\cba selects from frequent rules, and overlaps heavily without any restriction.

Our main competitor \ids behaves rather unstable because of the difficulty in hyper-parameter tuning.
One strong point of \ids is its small number of rules and conditions per rule.
However, a low model complexity leads to poor predictive performance.
By contrast, our model is capable of fixing a small number of rules, while still obtaining good performance with very few rules, as demonstrated in Figure \ref{fig:k}.
Through comprehensive comparison, we can confidently assert that our formulation for a diverse rule set
offers a significant advantage over \ids.

\begin{table}
  \caption{Predictive power. \ddsone is the rule set obtained in the first run of Algorithm \ref{alg:decisionset}}
  \label{tab:predictive}
\small
\setlength{\tabcolsep}{.5\tabcolsep}
\begin{tabular}{llrrrrrr}
\toprule
dataset & model &   $n_{\text{rules}}$ &  $n_{\text{conds}}$ &   bacc &    auc &   div & overlap\\
\midrule
iris	&\cba   &      7.00 &      2.23 &      0.90 &      0.92 &      0.89 &     28.67 \\
iris	&\cn   &      5.67 &      1.22 &      0.90 &      0.92 &      0.88 &    115.33 \\
iris	&\ids   & \bf{3.00} & \bf{1.00} &      0.83 &      0.88 & \bf{1.00} & \bf{0.00} \\
iris	&\ddsone &     10.67 &      2.21 & \bf{0.93} & \bf{0.95} & \bf{1.00} & \bf{0.00} \\
iris	&\dds &     10.67 &      2.12 & \bf{0.93} & \bf{0.95} & \bf{1.00} &      3.00 \\\midrule
contracept	&\cba   &      16.67 &      3.03 &      0.36 &      0.52 &      0.76 &    987.00 \\
contracept	&\cn   &      72.67 &      4.68 &      0.38 &      0.53 &      0.98 &   1177.33 \\
contracept	&\ids   & \bf{12.67} & \bf{1.00} &      0.32 &      0.49 &      0.88 &   1178.00 \\
contracept	&\ddsone &      15.33 &      5.26 &      0.37 &      0.53 & \bf{1.00} & \bf{2.00} \\
contracept	&\dds &      16.67 &      5.20 & \bf{0.40} & \bf{0.55} & \bf{1.00} &      5.33 \\\midrule
cardio	&\cba   &    100.00 &      4.84 &      0.56 &      0.68 &      0.86 &    1424.00 \\
cardio	&\cn   &     48.00 &      3.92 &      0.41 &      0.56 &      0.98 &    1695.00 \\
cardio	&\ids   & \bf{3.00} & \bf{1.00} &      0.33 &      0.50 &      0.13 &    1377.00 \\
cardio	&\ddsone &     47.00 &      9.10 &      0.67 &      0.75 & \bf{1.00} & \bf{35.00} \\
cardio	&\dds &     48.00 &      9.34 & \bf{0.71} & \bf{0.78} & \bf{1.00} &      80.67 \\\midrule
anuran	&\cba   &     100.00 &      6.08 &      0.54 &      0.72 &      0.73 &    3408.00 \\
anuran	&\cn   &      88.67 & \bf{4.15} &      0.49 &      0.70 &      0.98 &    5745.67 \\
anuran	&\ids   &        nan &       nan &       nan &       nan &       nan &        nan \\
anuran	&\ddsone & \bf{38.67} &      9.10 &      0.78 &      0.85 & \bf{1.00} & \bf{59.33} \\
anuran	&\dds &      39.67 &      8.50 & \bf{0.82} & \bf{0.87} & \bf{1.00} &     149.67 \\\midrule
avila	&\cba   & \bf{1.00} & \bf{2.00} &      0.25 &      0.50 & \bf{1.00} & \bf{0.00} \\
avila	&\cn   &     99.67 &      4.06 & \bf{0.29} & \bf{0.53} &      0.98 &  13052.33 \\
avila	&\ids   &       nan &       nan &       nan &       nan &       nan &       nan \\
avila	&\ddsone &      3.33 &      5.00 &      0.26 &      0.50 & \bf{1.00} & \bf{0.00} \\
avila	&\dds &      3.00 &      5.42 &      0.26 &      0.50 & \bf{1.00} & \bf{0.00} \\
\bottomrule\end{tabular}
\end{table}

\begin{figure*} 
\centering
\includegraphics[width=.99\textwidth]{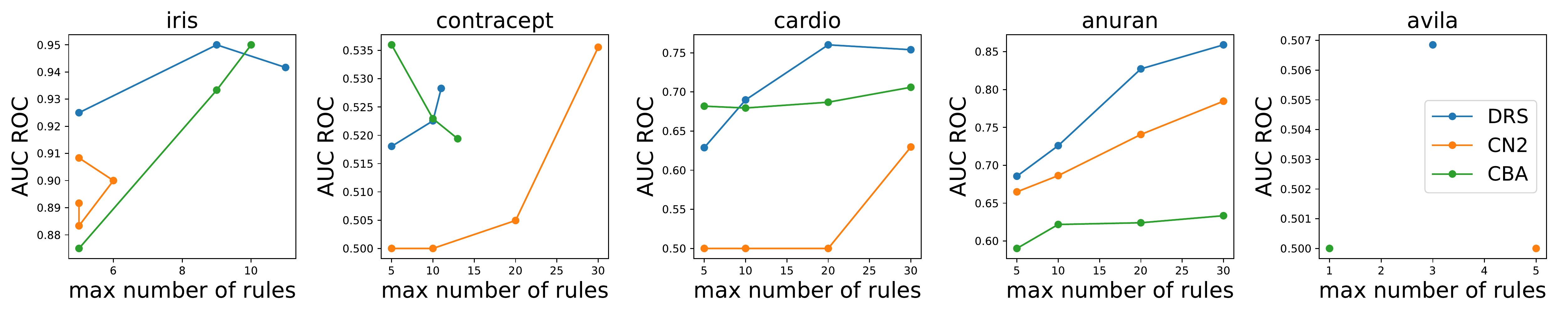} 
\caption{Predictive power with a limited number of rules.} \label{fig:k}
\Description{Predictive power with a limited number of rules.}
\end{figure*}

\subsection{Interpretability}

Interpretability is a very subjective criterion.
In our case, we try to quantify this quality via several dimensions.
It is common sense that a small number of rules, a small number of conditions and small overlap among rules facilitate interpretability.
Our model is excellent in small overlap.
The results in Figure \ref{fig:k} also verifies its qualification in the number of rules.
It is able to specify a small number of diverse rules, while still obtaining good predictive performance.
However, our model pays its price in the number of conditions.
This downside can be alleviated by a trade-off between model complexity and diversity.
When \lambdaq is relatively small, the selection of a new rule is not dominated by the diversity constraint.
Thus the discriminative measure in Equation (\ref{eq:discrim}) plays a leading role in choosing a rule, and it favors large-coverage rules, which implies a small number of conditions.

\section{Conclusion}\label{sec:conclusion}
In this work, we discuss desirable properties of a rule set in order to be considered interpretable.
Apart from a low model complexity, small overlap among decision rules has been identified to be essential.
Inspired by a recent line of work on diversification, we introduce a novel formulation for a diverse rule set that has both excellent discriminative power and diversity, and it can be optimized with an approximation guarantee.
Our model is indeed easy to interpret, as indicated by several interpretability properties.
Another major contribution in this work is an efficient sampling algorithm that directly samples decision rules that are discriminative and have small overlap, 
from an exponential-size search space, with a distribution that perfectly suits our objective.
Potential future directions include extensions to other notions of diversity.

\begin{acks}
We thank the UCI Machine Learning Repository \cite{Dua:2019} 
for contributing the datasets used in our experiments.
This research is supported by three Academy of Finland projects (286211, 313927, 317085),
the ERC Advanced Grant REBOUND (834862),
the EC H2020 RIA project ``SoBigData++'' (871042), and the
Wallenberg AI, Autonomous Systems and Software Program (WASP).
The funders had no role in study design, data collection and analysis, 
decision to publish, or preparation of the manuscript. 
\end{acks}

\bibliographystyle{ACM-Reference-Format}
\bibliography{references}

\newpage
\appendix
\section{Model parameters}\label{ap:param}
We propose several useful schemes to handle important hyper-parameters in our algorithm.
The number of selected rules $\nsel$ can be determined by a minimum marginal recall coverage $\epsilon$ (0.01 in our case) in the positive dataset.
In a multi-class case, the terminating criterion generalizes to every class of the dataset.
The trade-off term \lambdaq is required to be at least on the magnitude of the discriminative measure \quality to be able to enforce diversity.
The magnitude of \quality is dataset-dependent, and we estimate its magnitude by pre-sampling a set \samp of rules from the complete dataset using Algorithm \ref{alg:sampling}.
We design three intuitive modes for \lambdaq: \lambnone, \lambperm, and \lambstrict, with a value $0$, mean of quality \quality in~\samp, and maximum of quality \quality in~\samp, respectively, in an order of increasing diversity.
The number of sampled rules in each iteration, \nsamp, can be set rather freely.
A number larger than 100 gives a stable performance in our experiments.

\cba sets a minimum frequency to 0.1 and a minimum confidence to 0.3.
\cn searches rules with a beam width of 10 and a minimum covered examples of 15.
\ids is based on the public code released by its authors \footnote{\href{https://github.com/lvhimabindu/interpretable_decision_sets}{https://github.com/lvhimabindu/interpretable\_decision\_sets}}.
It has a high computational expense, which is quadratic to the number of candidate rules and is further exacerbated by slow convergence of local search.
Besides, it requires very heavy tuning, with 7 unbounded numerical hyper-parameters involved.
We set aside a validation set to tune its hyper-parameters each ranging from $[0,1000]$, as suggested in their paper \cite{lakkaraju2016interpretable}.
Since it relies on frequent rule mining, we are only able to set a minimum frequency to afford at most hundreds of candidate rules, due to its high complexity.
Therefore it inevitably falls short in some of larger datasets.

\end{document}